\documentclass[journal]{article}

\usepackage{amsmath,amsfonts,amsthm,dsfont, mathtools}

\usepackage{authblk}

\usepackage{hyperref}

\usepackage[left = 1in, right = 1in, top = 1in, bottom = 1in]{geometry}

\usepackage{graphicx}
\usepackage{epstopdf}
\usepackage{subfigure}
\usepackage{tikz}

\usepackage[toc,page]{appendix}

\newcommand\numberthis{\addtocounter{equation}{1}\tag{\theequation}}
\newtheorem{proposition}{Proposition}

\title{The Dependent Random Measures with Independent Increments in Mixture Models}

\author[1]{Cheng Luo\thanks{Email: 321luocheng@tongji.edu.cn}}
\author[2]{Richard Yi Da Xu\thanks{Email: YiDa.Xu@uts.edu.au}}
\author{Yang Xiang\thanks{Email: shxiangyang@tongji.edu.cn}}

\affil[1]{School of Electronic and Information Engineering, Tongji University}
\affil[2]{Faculty of Engineering and Information Technology, University of Technology, Syndey}

\begin{document}
\maketitle

\begin{abstract}
When observations are organized into groups where commonalties exist amongst them, the dependent random measures can be an ideal choice for modeling. One of the propositions of the dependent random measures is that the atoms of the posterior distribution are shared amongst groups,  and hence groups can borrow information from each other. When normalized dependent random measures prior with independent increments are applied, we can derive appropriate exchangeable probability partition function (EPPF), and subsequently also deduce its inference algorithm given any mixture model likelihood. We provide all necessary derivation and solution to this framework. For demonstration, we used mixture of Gaussians likelihood in combination with a dependent structure constructed by linear combinations of CRMs. 
Our experiments show superior performance when using this framework, where the inferred values including the mixing weights and the number of clusters both respond appropriately to the number of completely random measure used.

\textbf{Key words: Normalised dependent random measures; Bayesian Non-parameters; Linear mixed random measures; Mixture of Gaussians; Clustering.}
\end{abstract}

\section{Introduction}

Non-parametric statistical methods provide a very flexible framework for survival analysis and mixture models. Different from the classical Bayesian methods, the observations are assumed to be sampled from a probability random measure instead of a fixed probability distribution. One of the popular methods to define a probability random measure is using the normalization of a Completely Random Measure (CRM) \cite{Kingman1995}. To derive a CRM $\mu$, one needs to fix a base measure $H$ and define the measure $\mu(A)$ for all measurable sets $A$ to be a random variable and the random measures $\mu(A_1),...,\mu(A_n)$ are independent for any disjoint measurable sets $A_1,...,A_n$ \cite{Kingman1967}. The probability measure $\tilde{p}$ is hence defined by $\tilde{p} = \mu / \mu(\mathbb{X})$, where $\mathbb{X}$ is the space $\mu$ resides on. When the base measure $H$ is assumed to be a diffusion measure (a measure without atoms), the  sample drawn from $\tilde{p}$ are different almost surely. 

For observations organized in groups, a natural assumption is that some of the observations share the same atoms across groups whereas others do not. Obviously, defining a single random measure across all groups is insufficient in this setting. Consequently, we should define a vector of dependent random measures $(\mu_1,...,\mu_d)$ where observations at a group $i$ is associated with its corresponding random measure $\mu_i$. The pioneering work to consider this problem is \cite{MacEachern1999}. Using the stick-breaking paradigm of the Dirichlet Process, they proposed a dependent Dirichlet Process by assuming the random masses are shared by all groups and the locations are independent. The alternative method of defining random measures have been proposed since then, where most of them adopted the dependent L\'evy measures. By  constructing of a special L\'evy measure, Leisen et al \cite{Leisen2013} proposed a vector of Dirichlet Processes in a way that every margin of the dependent random measures is a Dirichlet Process. By virtue of  L\'evy copula \cite{Kallsen2006},  a powerful tool of defining dependence structure of L\'evy processes, \cite{Epifani2010} proposed a formula to define dependent random measures through fixed margins of the L\'evy measures and a L\'evy copula. Another intuitive and simple example is the linear combinations of CRMs. For instance, Griffin et al \cite{Griffin2013} proposed the Correlated Normalized Random Measures with Independent Increments (CNRMI) by introducing a binary matrix $\mathbf{Q}$ and constructing the dependent random measures $(\boldsymbol{\tilde{\mu}})$ through $\boldsymbol{\tilde{\mu}} = \mathbf{Q}\boldsymbol{\mu}$, where $\boldsymbol{\mu}$ is a vector of independent CRMs. Similarly, by changing $\mathbf{Q}$ to be a non-negative matrix, \cite{Chen2013} proposed the Linear Mixed Normalized Random Measures (LMNR). The dependent mixture models defined by \cite{Lijoi2014a} can also be seen as a very special example of this class.

Hence we feel the need to provide a framework of mixture models when the observations are formed in groups, and each group is associated with a random measure. However, in order for the inference algorithm to be practically implemented, a further assumption is required: For disjoint measurable sets $A_1,...,A_n$ the random vectors $\boldsymbol{\tilde{\mu}}(A_1),...,\boldsymbol{\tilde{\mu}}(A_n)$ are required to be independent. This assumption has a direct consequence that the increments for the dependent random measures are independent, and hence the name CNRMI was used in \cite{Griffin2013}.

Followed by the pioneering work of \cite{Ferguson1973}, Dirichlet Process is the first and one of most important stochastic processes introduced to Non-parametric Bayesian community. It is a special case of the class of normalized random measures. When the CRM $\mu$ is defined to be the Gamma subordinator \cite{Ferguson1973}\cite{Pitman1997}, the normalized random measure $\tilde{p} = \mu/\mu(\mathbb{X})$ becomes  Dirichlet Process. The posterior deduction and the inference of the Dirichlet Process is well studied by previous works, see \cite{Ferguson1973}\cite{Ishwaran2000}\cite{Neal2000}\cite{Walker2007} et al for detail. However, the posterior analysis for common normalized random measures is challenging. By introducing an auxiliary variable, \cite{James2008} proposed the first inference algorithm of sampling from a common normalized random measure. Under this framework, mixture models can be easily implemented when the prior is assumed to be the normalized $\sigma$-stable process and or the normalized generalized Gamma Process \cite{Lijoi2007}\cite{Favaro2013}. Using a similar methodology applied to the posterior analysis of the normalized random measures, one obtains the framework for the posterior analysis of dependent random measures with the help of the Exchangeable Partition Probability Function (EPPF). The work in \cite{Leisen2013}\cite{Lijoi2014} can be seen as special cases of this framework. 

In this paper, we summarize the posterior analysis of the Dependent Random Measures with Independent Increments (DRMI) and show how to apply this framework to the infinite mixture of Gaussians. We used the M as illustrative example of this framework, but it is, of course can be substituted by other conjugate mixture models. Different from the methods proposed by \cite{Griffin2013}\cite{Chen2013} our method resembles that of the Chinese Restaurant Process in a sense that, there is no need to incorporate a truncated number of activated atoms.

This paper is organized as follows. Section \ref{sec:theory} shows the EPPF of the DRMI. Section \ref{sec:inference} shows the inference of the DRMI and how to apply it to infinite mixture of Gaussians. In section \ref{sec:mcrm} we give the L\'evy measure of the Mixed Completely Random Measures (MCRM) and the details of the computation. The computational examples are given in section \ref{seq:exam} and the conclusion and future work are given in the last section.

\section{The dependent random measures with independent increments}\label{sec:theory}

Let $\mathbb{X}$ be a  separable completely  space, $\mathcal{B}(\mathbb{X})$ be the Borel $\sigma$-algebra defined on $\mathbb{X}$, and $(\tilde{\mu}_1,...,\tilde{\mu}_d)$ be a vector of CRMs defined on the measurable space $(\mathbb{X},\mathcal{B}(\mathbb{X}))$. The vector $(\mu_1,...,\mu_d)$ is called a \emph{dependent random measures with independent increments} if $\boldsymbol{\tilde{\mu}}(A_1),...,\boldsymbol{\tilde{\mu}}(A_n)$ are independent whenever the measurable sets $A_1,...,A_n$ are disjoint, where $\boldsymbol{\tilde{\mu}}(A_i) = (\tilde{\mu}_1(A_i),...,\tilde{\mu}_d(A_i))$ for $i=1,...,n$. 

Proved by \cite{Kingman1967}, a completely random measure can be decomposed into three parts: A fixed measure, a purely atomic random measure with finite atoms whose masses are random and locations are fixed and a completely random measure derived from a Poisson process. Using the L\'evy-Khintchine representation \cite{Sato}, the third part, say $\mu$, is determined by its L\'evy measure $\nu(ds)H(dx)$. If we ignore the first and second part for now, the Laplace functional of $\mu$ can be written as 
\begin{align*}
\mathbb{E}\left[e^{-\mu(f)}\right] = \exp\left\{-\int_{\mathbb{R}^+\times \mathbb{X}}(1 - e^{-sf(x)})\nu(ds)H(dx)\right\},
\end{align*}
where $\mu(f) = \int_{\mathbb{X}} f(x)\mu(dx) $, and $f$ is a measurable function almost surely. A multi-variational version  can be stated for DRMI that 
\begin{align}\label{eq:def_drmi}
\mathbb{E}\left[e^{-\sum_{i=1}^{d}\tilde{\mu}_i(f_i)}\right] = \exp\left\{-\int_{(\mathbb{R}^+)^d\times \mathbb{X}}(1 - e^{-\sum_{i=1}^{d}s_if_i(x)})\nu(ds_1,...,ds_d)H(dx)\right\},
\end{align}
where $f_i$ is measurable almost surely for $i=1,...,d$. 

Given a vector of DRMI, we can define a vector of \emph{normalized DRMI}, or NDRMI, as $\tilde{p}_i = \tilde{\mu}_i / \tilde{\mu}_i(\mathbb{X})$. The sample of a NDRMI is set of observations $\mathbf{X} = \{\boldsymbol{X}^{(1)},...,\boldsymbol{X}^{(d)}\}$, where $\boldsymbol{X}^{(i)} = \{X_{i,1},...,X_{i,n_i}\}$ and $X_{i,j}$ is a $\mathbb{X}$-valued random variable for $i=1,...,d$ and $j = 1,...,n_i$. Our assumption is that given a NDRMI, the random variables $X_{i,j}$ are independent with each other, or, a partially exchangeable proposition of the sample. Formally, let $C_{i,j} \in \mathcal{B}(\mathbb{X})$ be a measurable set for $i=1,...,d$ and $j = 1,...,n_i$, then the probability of the sample $\mathbf{X}$ is
\begin{align*}
\mathbb{P}(X_{i,j} \in C_{i,j}, ~ i =1,...,d,j=1,...,n_i) = \int \prod_{i=1}^{d}\prod_{j=1}^{n_i} \tilde{p}_i(C_{i,j}) \Phi(d\tilde{p}_1,...,d\tilde{p}_d),
\end{align*}
where $\Phi$ is the probability measure of $(\tilde{p}_1,...,\tilde{p}_d)$. When some of the $X_{i,j}$ take on the same value, for example, the sample $\mathbf{X}$ has $K$ distinct values $\{Y_1,...,Y_K\}$, and for each $k$ there are $q_{i,k}$ variables in group $i$ having the same value $Y_k$.  Suppose $C_1,...,C_K$ are disjoint measurable sets, then the probability of the sample can be rewritten as
\begin{align} \label{eq:EPPF_before}
\mathbb{P}(\mathbf{X}, Y_k \in C_{k}, ~ k=1,...,K)= \int \prod_{k=1}^K \prod_{i=1}^{d}\tilde{p}_i(C_k)^{q_{i,k}} \Phi(d\tilde{p}_1,...,d\tilde{p}_d).
\end{align}

The exchangeable partition probability function plays a key role in the Bayesian analysis of the mixture models. Pitman \cite{Pitman1995} gave a formal definition of the \emph{partially exchangeable probability function} and the EPPF for the Poisson-Dirichlet Process. In the field of the NDRMI, one of the major difference (with respect to a single random measure) is that the probability distribution of each partition is a function of the counts of all the groups. That is, for a special partition $k$, there is a joint density of $q_{i,k}$ for all the $i = 1,...,d$. Hence, the EPPF of the NDRMI is 
\begin{align}\label{eq:EPPF}
\Pi(K,\{n_i, q_{i,k}, ~ i = 1,...,d, k = 1,...,K\}) = \int_{\mathbb{X}^K} \mathbb{E}\left[\prod_{k=1}^{K} \prod_{i=1}^{d} \tilde{p}_i(dx)^{q_{i,k}}\right],
\end{align}
where $K$ is the number of partitions, $n_i$ is the number of observations in each group, and $q_{i,k}$ is the number of observations in group $i$ for partition $k$. To derive the expression of the EPPF of the NDRMI (\ref{eq:EPPF}), we only need to set $C_k \coloneqq C_{k,\epsilon} = \{y:d(Y_k,y) < \epsilon\}$ and let $\epsilon \downarrow 0$ in equation (\ref{eq:EPPF_before}). 

Following \cite{James2008} we substitute $\tilde{p}_i$ with $\tilde{\mu}_i/T_i$ where $T_i = \tilde{\mu}(\mathbb{X})$. Then we remove the denominator by introducing an auxiliary variable $U_i$ with the fact that 
\begin{align*}
T_i^{-n_i} = \frac{1}{\Gamma(n_i)} \int_0^\infty u^{n_i-1} e^{-T_i u_i} du_i.
\end{align*}
The auxiliary variables $u_1,...,u_d$ plays a key role in the expression of the EPPF. Through a few lines of deduction, we can show that the EPPF of the NDRMI can be stated by the following proposition, where the details of the proof is left in the appendix.
\begin{proposition}\label{prop}
Let any positive integers $K, n_1,...,n_d$ such that $K \leq \sum_{i=1}^{d} n_i $, the EPPF of the NDRMI is 
\begin{align*}
&\Pi(K,n_i, q_{i,k}, ~ i = 1,...,d, k = 1,...,K) = \\
&~~~~~~~~~~~~~~~~~~~~~~~~~\int_{(\mathbb{R}^+)^d} \left( \prod_{i=1}^{d} \frac{u_i^{n_i-1}}{\Gamma(n_i)}\right) e^{-\psi(u_1,...,u_d)} \left(\prod_{k=1}^{K}\tau_{\boldsymbol{q}_k}(u_1,...,u_d)  \right) du_1,...,du_d,
\end{align*}
where $\boldsymbol{q}_k = (q_{1,k},...,q_{d,k})$ and 
\begin{align*}
\psi(u_1,...,u_d) &= \int_{(\mathbb{R}^d)} (1 - e^{-\sum_{i=1}^{d}u_is_i}) \nu(ds_1,...,ds_d), \\
\tau_{\boldsymbol{q}_k}(u_1,...,u_d) &= \int_{(\mathbb{R}^+)^d} e^{-\sum_{i=1}^{d}s_iu_i} \prod_{i=1}^{d} s_i^{q_{i,k}} \nu(ds_1,...,ds_d).
\end{align*}
\end{proposition} 
In light of the Proposition \ref{prop}, we can give a closed form of the inference algorithm for the mixture model with the help of auxiliary variables. Different from the algorithms in \cite{Griffin2013}\cite{Chen2013}, this algorithm is a parallel to the Chinese restaurant process, where the random measures $\tilde{\mu}_1,...,\tilde{\mu}_d$ are integrated out.

\section{Inference}\label{sec:inference}

Suppose the sample $\mathbf{X}$ has $d$ groups and there are $n_i$ observations in group $i$. Suppose further the sample $\mathbf{X}$ has $K$ distinct values, and for each $k$ there are $q_{i,k}$ observations in group $i$ having the value $Y_k$. Now we want to know the probability of a new observation $X_{i,n_i+1}$ to be equal to $Y_k$ for $k=1,...,K$ and the probability of $X_{i,n_i+1}$ to be equal to some value new. By virtual of Proposition \ref{prop}, we can write 
\begin{align*}
\mathbb{P}(X_{i,n_i+1} = Y_k, u_1,...,u_d) &= \frac{\Pi(K, n_i + 1, q_{i,k} + 1,n_j, q_{j,k}, ~ j = 1,...,d, j \neq i, k = 1,...,K)}{\Pi(K,n_j, q_{j,k}, ~ j = 1,...,d, k = 1,...,K)} \\
&\propto \frac{\tau_{\boldsymbol{q}_k + \delta_i}(u_1,...,u_d) }{\tau_{\boldsymbol{q}_k}(u_1,...,u_d) } \numberthis \label{eq:sample_k}
\end{align*}
and 
\begin{align*}
\mathbb{P}(X_{i,n_i+1} = Y_*, u_1,...,u_d) &= \frac{\Pi(K+1, n_i + 1, q_{i,k+1} = 1, n_j, q_{j,k}, ~ j = 1,...,d, j \neq i, k = 1,...,K)}{\Pi(K,n_j, q_{j,k}, ~ j = 1,...,d, k = 1,...,K)} \\
&\propto \tau_{\delta_i}(u_1,...,u_d) , \numberthis \label{eq:sample_new}
\end{align*}
where $\delta_i$ is a binary vector of length $d$ with all the elements equal to $0$ but the $i$-th which is $1$, and $Y_*$ denotes a new value sampled from $H(dx)$. To draw $X_{i,j}$ for all pairs of $(i,j)$ we need to apply the exchangeability of the sample. We keep the status of all the other observations but $X_{i,j}$ and sample it from the conditional probability (\ref{eq:sample_k}) and (\ref{eq:sample_new}) and repeat this procedure for all $(i,j)$.

In addition to the observations, we need also to sample the values of the auxiliary variables $u_1,...,u_d$. From Proposition \ref{prop}, we can see that the density of $u_i$ is propositional to
\begin{align}\label{eq:update_u}
u_i^{n_i-1} \psi(u_1,...,u_d) \prod_{k=1}^{K}\tau_{\boldsymbol{q}_k}(u_1,...,u_d) .
\end{align}

In real applications, it is not wise to assume the observations are sampled directly from some random measure, instead, we assume the parameters of the model is distributed as some random measure and add a likelihood function to link the observations and the parameters. Formally, let $\Theta$ be a completely and separable space and $\mathcal{B}(\Theta)$ be the Borel $\sigma$-algebra defined on $\Theta$. Suppose $(\tilde{p}_1,...,\tilde{p}_d)$ is a NDRMI defined on $(\Theta, \mathcal{B}(\Theta))$, the model is constructed as follows:
\begin{equation}\label{eq:model_1}
\begin{aligned}
X_{i,j}|\theta_{i,j} &\sim f(\cdot|\theta_{i,j}) \quad i = 1,...,d, ~ j = 1,...,n_i,\\
\theta_{i,j} |(\tilde{p}_1,...,\tilde{p}_d) &\sim \tilde{p}_i(\cdot) \quad i = 1,...,d, ~ j = 1,...,n_i, \\
(\tilde{p}_1,...,\tilde{p}_d) &\sim \textrm{NDRMI},
\end{aligned} 
\end{equation}
where $f$ is the likelihood density function.

However, model (\ref{eq:model_1}) suffers from slow mixing since the parameters $\theta_{i,j}$ are moved one by another even if some of them are having the same value. Following \cite{Neal2000}, we add indicators $c_{i,j}$ for each $X_{i,j}$ and sample $c_{i,j}$ by virtual of equations  (\ref{eq:sample_k}) and (\ref{eq:sample_new}). Then the parameters $\theta_k$ can be sampled once with all the observations taking on parameter $\theta_k$. Formally, we modify model (\ref{eq:model_1}) to
\begin{equation} 
\begin{aligned}
X_{i,j}|c_{i,j},\theta_{1},...,\theta_K &\sim f(\cdot|\theta_{c_{i,j}}) \\
c_{i,j}|(\tilde{p}_1,...,\tilde{p}_d) &\sim \tilde{p}_i(\cdot) \\
\theta_k &\sim H(\cdot) \\
(\tilde{p}_1,...,\tilde{p}_d) &\sim \textrm{NDRMI},
\end{aligned}
\end{equation}
where $K$ is the number of the activated clusters. 
Combine all these facts together, the inference is stated as follows:
\begin{enumerate}
\item For $i=1,...,d$ and $j = 1,...,n_i$, leave $c_{i,j}$ alone and compute the frequencies $q_{i,k}$ for all the other clusters with $i=1,...,d$ and $k = 1,...,K$, and sample $c_{i,j}$ with probability proportional to 
\begin{align} \label{eq:updata_c}
\left\{\begin{array}{ll}
f(x_{i,j}|\theta_{k})\frac{\tau_{\boldsymbol{q}_k + \delta_i}(u_1,...,u_d) }{\tau_{\boldsymbol{q}_k}(u_1,...,u_d) }, &\textrm{with $c_{i,j} = k$}\\
\int f(x_{i,j}|\theta) H(d\theta) \tau_{\delta_i}(u_1,...,u_d), &\textrm{with $c_{i,j} = new$}
\end{array}\right.
\end{align}
\item For $k = 1,...,K$ update $\theta_k$ with density
\begin{align*}
\left(\prod_{c_{i,j} = k} f(x_{i,j}|\theta) \right)H(d\theta)
\end{align*}
\item For $i = 1,...,d$, update $u_i$ with density (\ref{eq:update_u}).
\end{enumerate}

\section{The LMRM}\label{sec:mcrm}

According to the inference algorithm summarized in the last section, once the L\'evy measure $\nu(ds_1,...,ds_d)$ of the unnormalized dependent random measures $(\tilde{\mu}_1,...,\tilde{\mu}_d)$ is determined, the functions $\tau_{\boldsymbol{q}_k}(u_1,...,u_d)$ and $\psi(u_1,...,u_d)$ can be expressed analytically, and consequently, the updating of $c_{i,j}$ and $u_i$ can be derived. Hence we focus on the computation of the L\'evy measures. Following \cite{Chen2013}, the LMRM is constructed by linear combinations of CRMs. Formally, let $\mu_1,...,\mu_R$ be independent CRMs and $w_{i,r}$ be non-negative numbers for $i =1,..,d$ and $r = 1,...,R$. The LMRM is constructed by
\begin{align*}
\tilde{\mu}_i = \sum_{r=1}^{R} w_{i,r} \mu_r, \quad i = 1,...,d.
\end{align*}
It can be seen that when we set $w_{i,r} \in \{0,1\}$  the model of \cite{Griffin2013} is restored and when we set $R = 3$ and $w_{1,1} = w_{1,3} = 1, w_{1,2} = 0$ and $w_{2,2} = w_{2,3} = 1,w_{2,1} = 0$, the model of \cite{Lijoi2014a} is restored. However, to use the inference algorithm in the last section, we should give an explicit expression of the L\'evy measure of the LMRM. 

We further assume all the CRMs $\mu_1,...,\mu_R$ have the same L\'evy measure $\nu^*(ds)$, or, the direction of the LMRM. Then for any almost surely measurable functions $f_1,...,f_d$, we have the Laplace functional 
\begin{align*}
\mathbb{E}\left[e^{-\sum_{i=1}^{d}\tilde{\mu}_i(f_i)}\right] = \exp\left\{-\int_{\mathbb{R}^+\times\mathbb{X}}\sum_{r=1}^{R}(1-e^{-s\sum_{i=1}^{d}w_{i,r}f_i(x)}) \nu^*(ds)H(dx)  \right\}.
\end{align*}
Now we construct a measurable function 
\begin{align*}
g(w_1,...,w_d) = \left\{\begin{array}{ll}
1, &\textrm{if }(w_1,...,w_d) = (w_{1,r},...,w_{d,r}), r = 1,...,R,\\
0, &\textrm{otherwise},
\end{array}\right.
\end{align*}
and rewrite the Laplace functional of $(\tilde{\mu}_1,...,\tilde{\mu}_d)$ as
\begin{align*}
\mathbb{E}\left[e^{-\sum_{i=1}^{d}\tilde{\mu}_i(f_i)}\right] = \exp\left\{-\int_{\mathbb{R}^+\times\mathbb{X}}\int_{(\mathbb{R}^+)^d} (1-e^{-s\sum_{i=1}^{d}w_if_i(x)})g(dw_1,...,dw_d) \nu^*(ds)H(dx)  \right\} .
\end{align*}
If we substitute $s_i = w_i s$, the above equation is changed to
\begin{align*}
\mathbb{E}\left[e^{-\sum_{i=1}^{d}\tilde{\mu}_i(f_i)}\right] = \exp\left\{-\int_{\mathbb{R}^+\times\mathbb{X}}\int_{\mathbb{R}^+} (1-e^{-\sum_{i=1}^{d}s_if_i(x)})g(ds_1/s,...,ds_d/s) \nu^*(ds)H(dx)  \right\} .
\end{align*}
This gives us the L\'evy measure of $(\tilde{\mu}_1,...,\tilde{\mu}_d)$ as
$$
\nu(ds_1,...,ds_d) = \int_{\mathbb{R}^+} g(ds_1/s,...,ds_d/s) \nu^*(ds) = \sum_{r = 1}^{R} \int \mathds{1}_{w_{1,r}s,...,w_{d,r}s}(ds_1,...,ds_d) \nu^*(ds).
$$

The remaining is to show that for any disjoint measurable sets $A$ and $B$ the random vector $(\tilde{\mu}_1(A),...,\tilde{\mu}_d(A))$ and $(\tilde{\mu}_1(B),...,\tilde{\mu}_d(B))$ are independent. It is suffices to show that
\begin{align*}
\mathbb{E}\left[e^{-\sum_{i=1}^{d}\tilde{\mu}_i(A) - \sum_{i=1}^{d}\tilde{\mu}_i(B)}\right] = \mathbb{E}\left[e^{-\sum_{i=1}^{d}\tilde{\mu}_i(A)}\right] \mathbb{E}\left[e^{-\sum_{i=1}^{d}\tilde{\mu}_i(B)}\right].
\end{align*}
But this is obvious since 
\begin{align*}
\mathbb{E}\left[e^{-\sum_{i=1}^{d}\tilde{\mu}_i(A) - \sum_{i=1}^{d}\tilde{\mu}_i(B)}\right] &= \mathbb{E}\left[e^{-\sum_{i=1}^{d}\sum_{r= 1}^{R} w_{i,r}\mu_i(A) - \sum_{i=1}^{d}\sum_{r= 1}^{R} w_{i,r}\mu_i(B)}\right] \\
&= \mathbb{E}\left[e^{-\sum_{i=1}^{d}\sum_{r= 1}^{R} w_{i,r}\mu_i(A)}\right] \mathbb{E}\left[e^{-\sum_{i=1}^{d}\sum_{r= 1}^{R} w_{i,r}\mu_i(B)}\right]
\end{align*}
The last equation follows from the fact that $\mu_1,...,\mu_R$ are independent CRMs.

\subsection{The Gamma direction}
For a concrete example, we give the expression of $\tau_{q_{1,k},...,q_{d,k}}(u_1,...,u_d)$ and $\psi(u_1,...,u_d)$ when $\nu^*(ds)$ is the L\'evy measure of the Gamma  process, or
\begin{align}\label{eq:gamma_measure}
\nu^*(ds) = \alpha s^{-1}e^{-s}.
\end{align}
Recall that
\begin{align*}
\tau_{\boldsymbol{q}_k}(u_1,...,u_d) &= \int_{(\mathbb{R}^+)^d} e^{-\sum_{i=1}^{d}s_iu_i} \prod_{i=1}^{d} s_i^{q_{i,k}} \nu(ds_1,...,ds_d)
\end{align*}
and substitute 
$$
\nu(ds_1,...,ds_d) = \sum_{r=1}^{R} \int_0^\infty \mathds{1}_{w_{1,r}s,...w_{d,r}s}(ds_1,...,ds_d) \alpha s^{-1}e^{-s} ds
$$
gives the expression of $\tau_{\boldsymbol{q}_k}(u_1,...,u_d)$ which is
\begin{equation}\label{eq:diri_tau}
\tau_{\boldsymbol{q}_k}(u_1,...,u_d) = \Gamma(t_k)\sum_{r=1}^R  \frac{\prod_{i=1}^{d}w_{i,r}^{q_{i,k}}}{(h_r + 1)^{t_k}},
\end{equation}
where $t_k = \sum_{i=1}^{d}q_{i,k}$ denotes the number of observations taking on cluster $k$ across groups and $h_r = \sum_{i=1}^{d}w_{i,r}u_i$. Similarly, we have
\begin{equation}\label{eq:diri_psi}
\psi(u_1,...,u_d)  = \sum_{r=1}^{R} \log(h_r + 1).
\end{equation}
Insert the result of equation (\ref{eq:diri_tau}) into the updating formula of $c_{i,j}$ in equation (\ref{eq:updata_c}) gives the probability of $c_{i,j}$. The probability of $c_{i,j} = k$ for $k = 1,...,K$ is proportional to 
\begin{align*}
t_k f(x_{i,j}|\theta_k) \left(\sum_{r=1}^R  \frac{w_{i,r}^{q_{i,k}+1}\prod_{l=1,l\neq i}^{d}w_{l,r}^{q_{l,k}}}{(h_r + 1)^{t_k}}\right)\left(\sum_{r=1}^R  \frac{\prod_{l=1}^{d}w_{l,r}^{q_{l,k}}}{(h_r + 1)^{t_k}}\right)^{-1},
\end{align*}
and the probability of $c_{i,j}$ is not equal to any $k$ is proportional to 
\begin{align*}
\alpha \int f(x_{i,j}|\theta)H(d\theta)\sum_{r=1}^{R} \frac{w_{i,r}}{w_{i,r}u_i + 1}.
\end{align*}
Comparing with the updating formulas in the Dirichlet Process, the only difference is that there is an additional term in each of the above expressions. Besides this, the Dirichlet Process updating formula is restored if we simply set $R = 1$\footnote{We should also integrate out the auxiliary variables $u_i$, and that is possible when we set $R = 1$.}. In fact, the LMRM can be simplified into normalized random measures when $R = 1$, and how to derive the updating formulas of Dirichlet Process from normalized Gamma Process can be found in \cite{Favaro2013}.

The update of the auxiliary variables $u_1,...,u_d$ and the mixing weights $w_{i,r}$ for $i=1,...,d$ and $r = 1,...,R$ can be derived using the stochastic gradient MCMC proposed in \cite{Ma2015}. For reader's reference, we give the gradients of $u_i$ and $w_{i,r}$ in the below,
\begin{align*}
\frac{\partial p}{\partial u_i} &= \frac{n_i-1}{u_i} - \alpha \sum_{r=1}^{R} \frac{w_{i,r}}{h_r+1} - \sum_{k=1}^{K} \frac{\tau_{\boldsymbol{q}_k + \delta_i}(u_1,...,u_d)}{\tau_{\boldsymbol{q}_k}(u_1,...,u_d)},\\
\frac{\partial p}{\partial w_{i,r}} &= -\alpha \frac{u_i}{h_r + 1} + \sum_{k=1}^{K} \left(\frac{q_{i,k}}{w_{i,r}} - \frac{t_ku_i}{h_r + 1}\right) \left(\sum_{r'=1}^{R} \prod_{i=1}^{d} \left(\frac{w_{i,r'}}{w_{i,r}} \right)^{q_{i,k}}\left(\frac{h_{r'} + 1}{h_r+1}\right)^{t_k}\right),
\end{align*} 
where $p$ is the joint density of $u_{i}, \tau_{\boldsymbol{q}_k}(u_1,...,u_d)$ and $\psi(u_1,...,u_d)$. 
Another computational problem should be noted is that the  computation of $\tau_{\boldsymbol{q}_k}(u_1,...,u_d)$ often underflow since $h_r + 1 > 1$ and the frequency $t_k$ can be hundreds even thousands which depends on the size of the sample, hence the term $(h_r + 1)^{-t_k} \downarrow 0$ quickly. To solve this problem, we need to compute the fraction  $\tau_{\boldsymbol{q}_k  +\delta_i}(u_1,...,u_d) / \tau_{\boldsymbol{q}_k}(u_1,...,u_d)$ directly instead of working out the terms one by one. With some algebra, we have
\begin{align*}
\frac{\tau_{\boldsymbol{q}_k + \delta_i}(u_1,...,u_d)}{\tau_{\boldsymbol{q}_k}(u_1,...,u_d)} = t_{k} \sum_{r=1}^R \frac{w_{i,r}}{(h_r + 1)\sum_{r'=1}^{R} \left(\frac{h_r+1}{h_{r'}+1}\right)^{t_k} \prod_{i=1}^d \left(\frac{w_{i,r'}}{w_{i,r}}\right)^{q_{i,k}}}.
\end{align*}

\section{Illustration examples}\label{seq:exam}

Infinite Gaussian is one of the most important model in the clustering methods. The first infinite Gaussian model is an application of the Dirichlet Process \cite{Rasmussen2000}. This model extent the traditional finite mixture of Gaussians into a more flexible infinite mixtures of Gaussians, and thus remove the constraint that the number of Gaussians should be fixed in advance. Neal \cite{Neal2000} gave a full description of the mixture model with Dirichlet Process and the inference algorithms, including the original and accelerated form, the conjugate base measure and the non-conjugate base measure. The infinite Gaussians can also be simplified into the infinite k-means, which fixed the standard deviations and set $K\uparrow\infty$, see \cite{Kulis2011}. 

However, all of the above examples assume the base measure in the Dirichlet Process is a diffusion measure, then the atoms are different almost surely. When observation are organized in groups and we want some of the atoms are shared across groups, a dependent random measures should be applied. The hierarchical Dirichlet Process (HDP) \cite{Cai2013} is a famous model of tackling this problem. In this model, a discrete measure $\mu$ is first drawn from a Dirichlet Process, then the dependent measures $\tilde{\mu}_1,...,\tilde{\mu}_d$ are sampled from a Dirichlet Process with base measure $\mu$. To see a concrete example for the clustering with HDP, the readers can refer to \cite{Wang2013}. By definition, the NDRMI has such a proposition as well. In our example, we show how to cluster groups of observations with LMRM. Our example is inspired by Lijoi et al \cite{Lijoi2014a} which defined the mixed random measure as 
\begin{align*}
\tilde{\mu}_1^* &= \mu_0 + \mu_1, \\
\tilde{\mu}_2^* &= \mu_0 + \mu_2.
\end{align*}
They constructed the synthetic data set by assume $\mu_r$ to be a finite mixture of Gaussians for $r = 0,1,2$ and hence both $\tilde{\mu}_1^*$ and $\tilde{\mu}_2^*$ are mixtures of Gaussians. Since this model is just a special case of the LMRM, we use a similar synthetic data set and reasons will be explained below.

\subsection{Synthetic data}

We first simulate the CRMs by setting $\mu_1 = 1/2\mathcal{N}(-10,1) + 1/2\mathcal{N}(-5,1)$, $\mu_2 = 1/2\mathcal{N}(0,1) + 1/2\mathcal{N}(5,1)$ and $\mu_3 = 1/2\mathcal{N}(10,1) + 1/2\mathcal{N}(15,1)$, where $\mathcal{N}(a,b)$ denotes the Normal distribution with mean $a$ and standard deviation $b$. The $i$-th group of observations are sampled from 
\begin{align*}
\tilde{\mu}_i = w_{i,1} \mu_1 + w_{i,2} \mu_2 + w_{i,3} \mu_3, \quad i = 1,...,d.
\end{align*}
In our experiment, we set $d = 2$ and the size of the two groups are equal, which is $300$. For a better demonstration of the LMRM, we set $w_{1,1} = 0.3, w_{1,2} = 0.01, w_{1,3} = 0.69, w_{2,1} = 0.3, w_{2,2} = 0.69, w_{2,3} = 0.01$. See Figure \ref{fig:hist_data} the histogram of the sampled observations. The setting of the parameter $w_{i,r}$ is based on the following consideration. Since the LMRM $\tilde{\mu}_1$ and $\tilde{\mu}_2$ are integrated out in the inference algorithm (and consequently the CRMs $\mu_1,\mu_2,u_3$), we cannot recover them directly. Thus the weights $w_{i,r}$ ($i=1,...,d,r= 1,...,R$) do not need to be equal to the true values, and hence the inferring results are ambiguous even if the cluster centers are recovered. However, by setting one the mixing weights to be extreme values (close to $0$), at least one of the inferred mixing weights must be small enough to suppress the corresponding completely random measure. The effect of $R$ can also be seen by this setting, because when $R$ is smaller than the true value, it is obvious that no mixing weights will be suppressed. 

\begin{figure}
\centering
\begin{subfigure}
{\includegraphics[width = 0.45\textwidth]{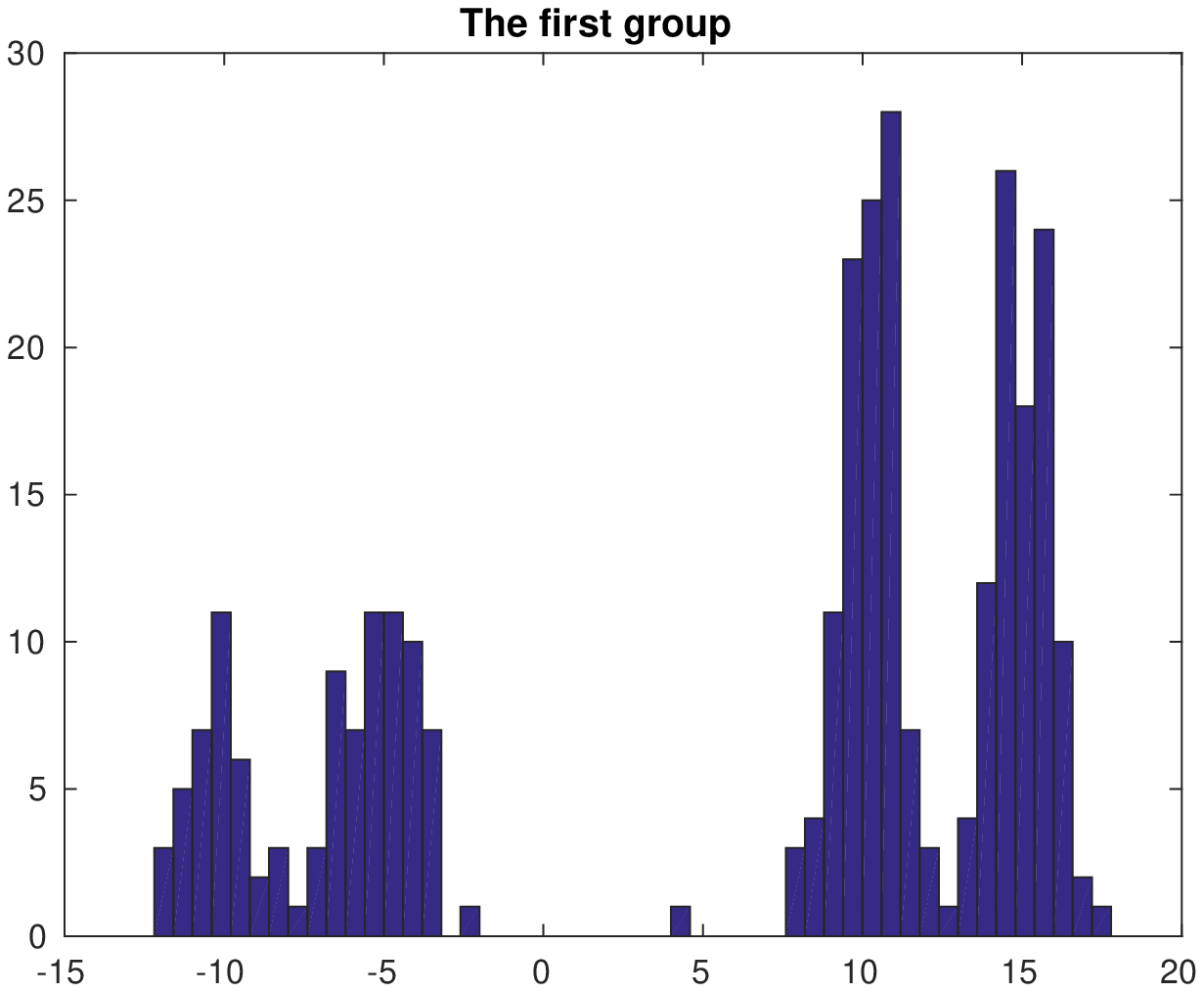}}
\end{subfigure}
\begin{subfigure}
{\includegraphics[width = 0.45\textwidth]{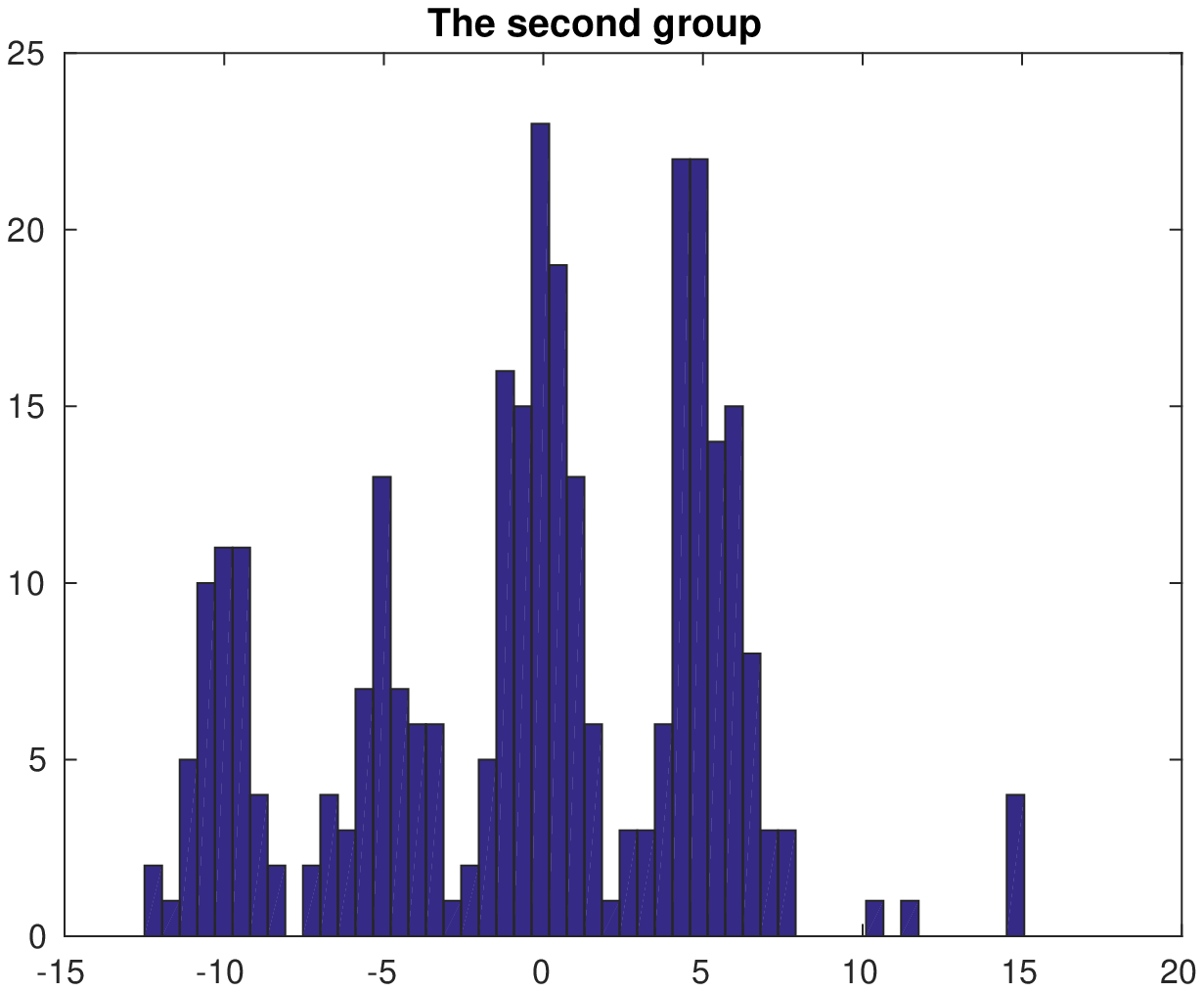}}
\end{subfigure}
\caption{The histogram of the two groups of observations. We can see that the first two clusters are shared by them, but the third and the fourth clusters are suppressed in the first group while the last two clusters are suppressed by the second group.} \label{fig:hist_data}
\end{figure}

In our experiments, we set $R$ to be $2,3$ and $4$ and observe the values of the mixing weights $w_{i,r}$, and the activated number of clusters. Initially, we set $w_{i,r} = u_i = 1$ for all of them, the standard deviation for each cluster is set to $\sigma_1 = 1$, and the base measure is set to be $\mathcal{N}(0,2.6)$, where $2.6$ is the approximate standard deviation for the sample. The parameter in the L\'evy measure of the Gamma Process is set to $\alpha = 0.005$. 

We test the performance of the model by running Gibbs sampling for 10000 iterations and discard the first 2000 iterations as burn in. The average values of the mixing weights and the number of clusters are shown in Table \ref{tab:r=3}. It can be seen that when $R = 3$ (the second row), the average number of clusters are approximately $6$, and one of the weights is very small in each group as our expectation. Further more, the proportions of the weights in each group are approximately equal to the true proportions ($0.01:0.35:0.69$). We also need to note that the weights $w_{1,1}$ is the largest in group 1 and $w_{2,1}$ is the smallest in group 2, indicating that the inferred first completely random measure is corresponding to $\mu_3$. Similarly, the  smallest and largest mixing weight in group 1 and 2 is $w_{1,2}$ and $w_{2,2}$, and this means the inferred second completely random measure is corresponding to $\mu_2$.  This fact complies with our assumption that the first group ignores one completely random measure while the second group ignores another one. 

When we set $R = 4$ (the third row of Table \ref{tab:r=3}), the average number of activated clusters is also approximately to equal to the true value. Similarly, in each group, there is one of the weights near $0$, and that means our inference correctly detects the structure of the data. As the same with $R=3$, the largest values and the smallest values are paired ($w_{1,1}$ v.s. $w_{2,1}$ and $w_{1,3}$ v.s. $w_{3,1}$).

When we set $R = 2$. Our inference cannot detect the structure of the sample since there are fewer assumed CRMs than the real settings. However, the results are much more interesting. Because it can be inferred that our algorithm split the shared completely random measure into two and allocate them to the two groups. Firstly, we need to note that the average number of activated clusters is approximately $8$, and this is the first clue. Then we dive into the inferred centers and the percentages of each cluster in each group. In Table \ref{tab:r=2}, we can see that the clusters in the first completely random measure are split. Cluster $3$ and 8 are actually one cluster and cluster 5 and 7 are another one. In fact, these two clusters are just the those in the first completely random measure and they are split so group 1 is assigned cluster 3 and 7 while group 2 is assigned cluster 5 and 8.

Combine these facts, we can see that when $R$ is set to be greater or equal to the real value of the sample, the algorithm can detect the right structure. But when $R$ is set to be too small, the algorithm will split the shared clusters and create redundant ones to suit the settings. 
\begin{table} 
\centering
\caption{The averages of the mixing weights and the number of activated clusters.}\label{tab:r=3}
\begin{tabular}{c|c|c|c|c|c|c|c|c|c}
\hline
&$w_{1,1}$ &$w_{1,2}$  &$w_{1,3}$  &$w_{1,4}$  &$w_{2,1}$  &$w_{2,2}$ &$w_{2,3}$  &$w_{2,4}$  &$K$  \\
\hline
$R = 2$ &2.1563 &0.0743 &  &  &0.0504  &2.0894 & &  &8.0785 \\
\hline
$R = 3$ &2.0379 &0.0635  &1.0038 &  &0.0346 &1.6804 &0.8997 & &6.0480 \\
\hline
$R = 4$ &2.3880 &1.0882 &0.0351 &1.0882 &0.0503 &1.3252 &2.3666  &11.3251 &6.1414\\
\hline
\end{tabular}
\end{table}

\begin{table}
\centering
\caption{The clusters and the percentages of each group.}\label{tab:r=2}
\begin{tabular}{c c | c c | c c }
\hline
& &\multicolumn{2}{c}{Group 1}  &\multicolumn{2}{c}{Group 2} \\
\hline
Cluster &Mean  &Count &Percent  &Count &Percent \\
\hline
   1    &5.0820      &1      &0.33\%    &104    &34.67\%\\
   2    &15.0416     &96     &32.00\%   &1      &0.33\%\\
   3    &-9.8725     &52     &17.33\%   &3      &1.00\%\\
   4    &-0.0499     &5      &1.67\%    &91     &30.33\%\\
   5    &-5.0027     &4      &1.33\%    &55     &18.33\%\\ 
   6    &9.9423      &104    &34.67\%   &3      &1.00\%\\
   7    &-5.0607     &38     &12.67\%   &2      &0.67\%\\
   8    &-10.0084    &0      &0.00\%    &41     &13.67\%\\
\hline
\end{tabular}
\end{table}

\subsection{Real data set}

Besides synthetic data set, we also test our algorithm in a clinical data set.  This data consists of drug information collected on 50 patients  used to perform frequency and descriptive statistics\footnote{The data set is available  \url{http://calcnet.mth.cmich.edu/org/spss/Prj_New_DrugData.htm}}.This data set comprises of 6 variables: Subject, Treatment, Age, Gender, Before\textunderscore exp\textunderscore BP and After\textunderscore exp\textunderscore BP, where Treatment is a binary variable with  1 for treatment and  0 for placebo, and Before\textunderscore exp\textunderscore BP and After\textunderscore exp\textunderscore BP are for blood pressure before and after experiment respectively. We construct the first group using the blood pressure for patients taking real pills and the second group for placebo. Figure \ref{fit:real_hist} shows the histograms of the two groups. It can be seen that the first group consists of three clusters with centers are approximately $85,95,105$ while the second group consists of two clusters with centers are approximately $91,98$. That is to say that one of the clusters is shared by the two groups. In fact, most of the patients have similar blood pressures (95 v.s. 98) before experiments except for a few people having extremely high blood pressure (105 to 115). But after the experiments, the patients taking real pills have a lower blood pressure than those having placebos (85 v.s. 91). 

We assume the observations are distributed as mixtures of Gaussians like the synthetic data. The standard deviation for each cluster is assumed to be 3 and the standard deviation for the base measure is assumed to be 6. We fix the concentration parameter $\alpha = 0.005$ and run the Gibbs sampling for 10000 iterations. The average values of the mixing weights are shown in Table \ref{tab:real_w}. It is clear that when we set $R=3$ or $R=4$, one of the mixing weight in group 2 ($w_{2,1}$ and $w_{2,3}$) is close to 0, meaning there is one completely random measure being ignored in the second group. However, when we set $R = 2$, all the mixing weights are relatively large, indicating that this setting cannot discover the structure of the sample. 

The centers and the percentages of each cluster is shown in Table \ref{tab:real_r=2}, Table \ref{tab:real_r=3} and \ref{tab:real_r=4} for $R=2$, $R=3$ and $R=4$ respectively. It is obvious that there is one cluster shared by the groups when we set $R=3$ and $R=4$. On the contrary, when we set $R=2$ the clusters mixed up and the true structure of the sample is not recovered.

\begin{figure}
\centering
\subfigure{\includegraphics[width = 0.45\textwidth]{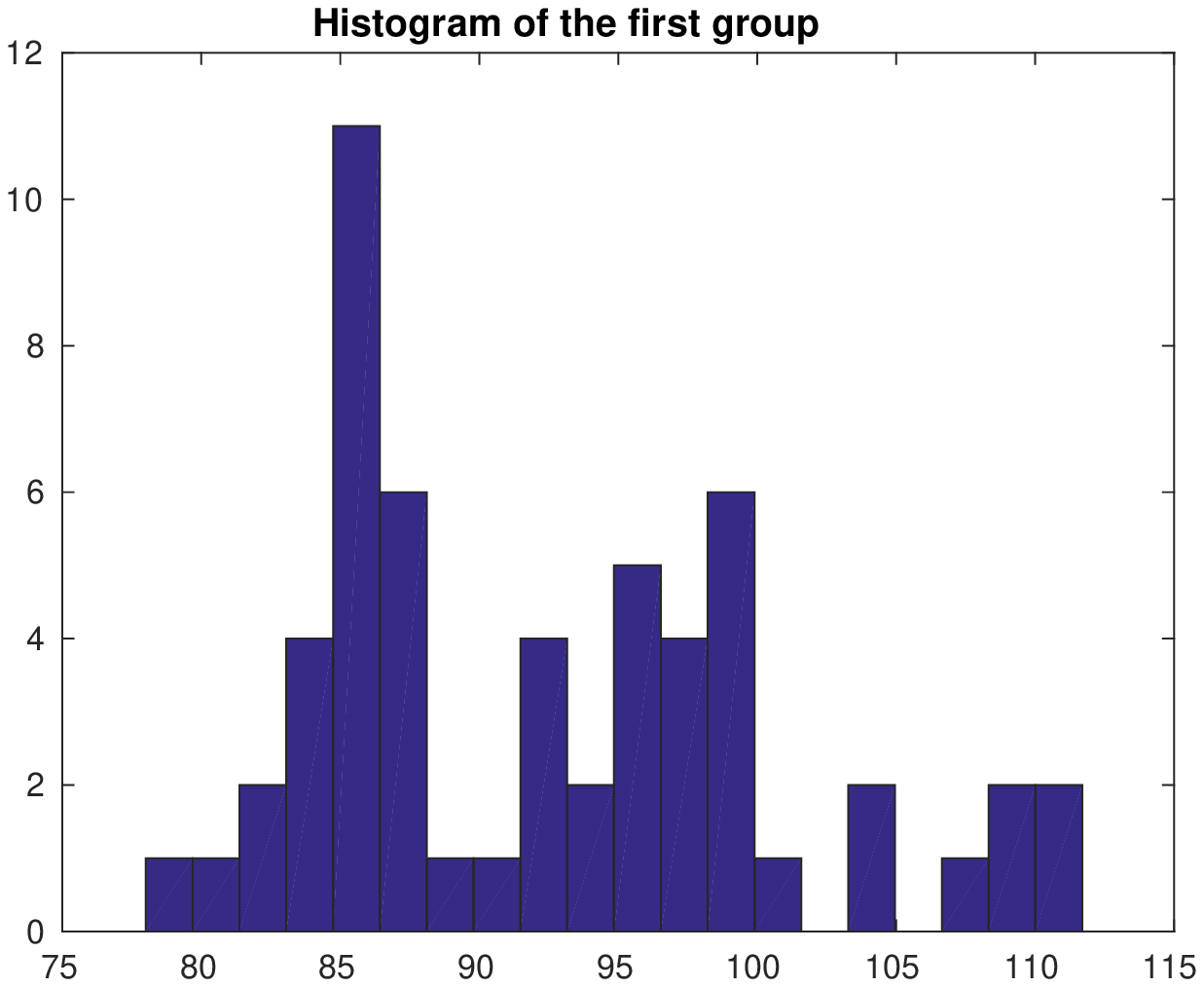}}
\subfigure{\includegraphics[width = 0.45\textwidth]{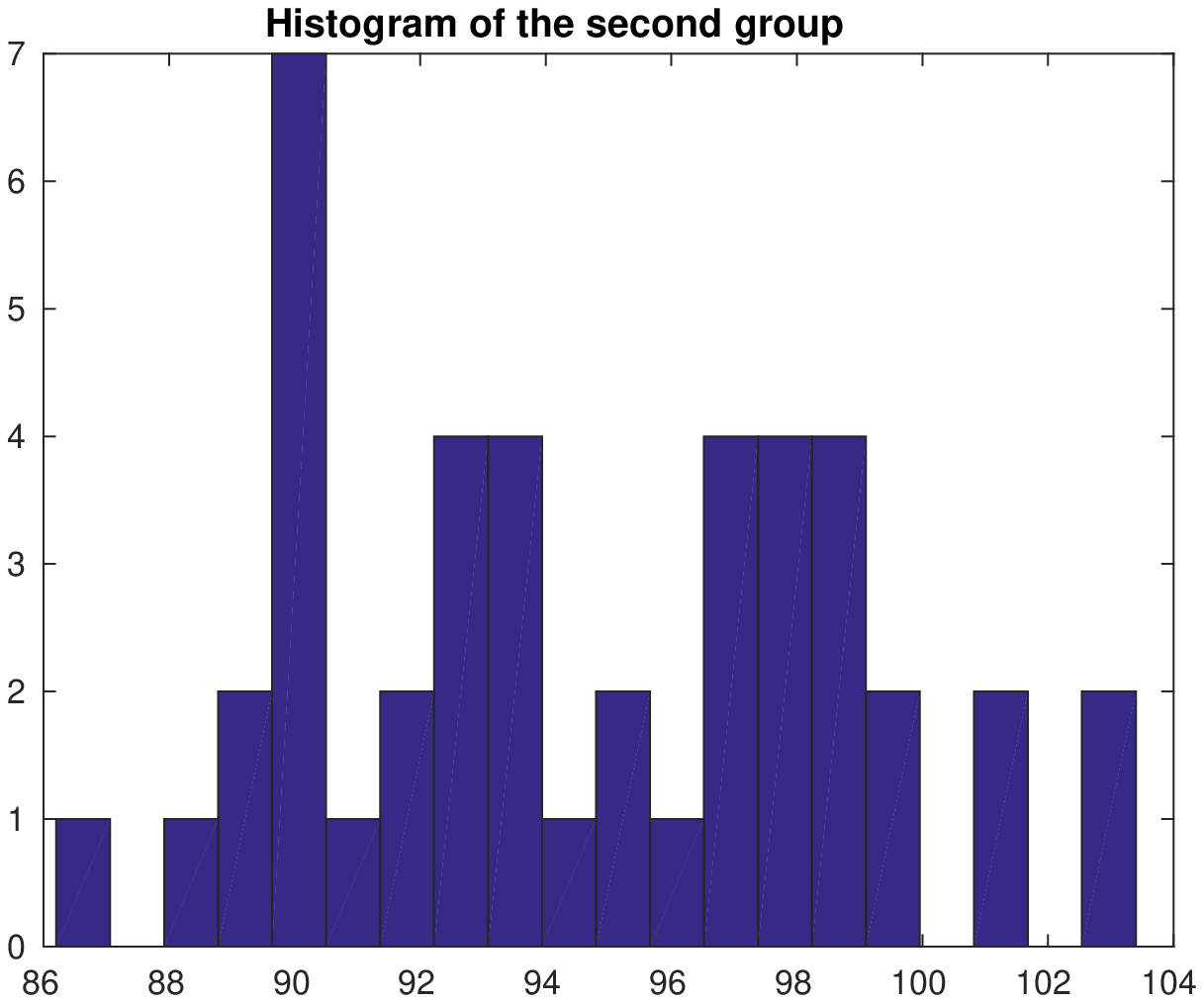}}
\caption{The histogram of the clinical data set. It can be seen that the first group consists of three clusters with centers approximately $85,95,105$ while the second group consists of two clusters with centers $91,98$.} \label{fit:real_hist}
\end{figure}

\begin{table}[!ht]
\centering
\caption{The average values of the mixing weights.}\label{tab:real_w}
\begin{tabular}{c|c|c|c|c|c|c|c|c|c}
\hline
 &$w_{1,1}$ &$w_{1,2}$ &$w_{1,3}$ &$w_{1,4}$ &$w_{2,1}$ &$w_{2,2}$ &$w_{2,3}$ &$w_{2,1}$ &$K$\\
 \hline
 $R=2$ & 0.9590 & 0.9300 & & & 6.0546 & 6.1913 & & & 3.0733\\ 
 \hline
 $R=3$ & 1.4742 & 0.3496 & 0.6308 & & 0.0012 & 1.1795 & 0.9358 & & 4.0183\\
 \hline
 $R=4$     & 0.5355    & 0.8963    & 1.4557    & 0.5148
     &  1.1265    & 3.6203    & 0.0000    & 1.0332 & 4.0251 \\
 \hline 
\end{tabular}
\end{table}

\begin{table}
\centering
\caption{The cluster information for $R = 2$}\label{tab:real_r=2}
\begin{tabular}{cc|cc|cc}
\hline
& &\multicolumn{2}{c}{Group 1}  &\multicolumn{2}{c}{Group 2} \\
\hline
Cluster &Mean  &Count &Percent  &Count &Percent \\
\hline
   1    &96.4534     &24     &42.86\%   &34     &77.27\%\\
   2    &86.5541     &27     &48.21\%   &10     &22.73\%\\
   3    &109.6400    &5      &8.93\%    &0      &0.00\%\\
\hline
\end{tabular}
\end{table}

\begin{table}
\centering
\caption{The cluster information for $R = 3$}\label{tab:real_r=3}
\begin{tabular}{cc|cc|cc}
\hline
& &\multicolumn{2}{c}{Group 1}  &\multicolumn{2}{c}{Group 2} \\
\hline
Cluster &Mean  &Count &Percent  &Count &Percent \\
\hline
   1    &92.3094     &6      &10.71\%    &26     &59.09\%\\
   2    &109.6400    &5      &8.93\%     &0      &0.00\%\\
   3    &85.3741     &27     &48.21\%    &0      &0.00\%\\
   4    &98.2722     &18     &32.14\%    &18     &40.91\%\\
\hline
\end{tabular}
\end{table}

\begin{table}
\centering
\caption{The cluster information for $R = 4$}\label{tab:real_r=4}
\begin{tabular}{cc|cc|cc}
\hline
& &\multicolumn{2}{c}{Group 1}  &\multicolumn{2}{c}{Group 2} \\
\hline
Cluster &Mean  &Count &Percent  &Count &Percent \\
\hline
   1    &97.3745     &20     &35.71\%    &27     &61.36\%\\
   2    &108.5833    &6      &10.71\%    &0      &0.00\%\\
   3    &85.5393     &28     &50.00\%    &0      &0.00\%\\
   4    &90.6211     &2      &3.57\%     &17     &38.64\%\\
\hline
\end{tabular}
\end{table}

\section{Conclusion and future work}\label{sec:conclusion}

In this paper, we have proposed  a framework for modeling mixture models when observations are organized in groups and the prior is a NDMRI.  We pointed out that when the L\'evy measure of the NDRMI is given, the EPPF can be derived analytically and hence the inference resembles that of the product of the Chinese restaurant process with  additional terms  $(u_1,...,u_d)$. As a special case, we have derived the L\'evy measure of the LMRM and applied the inference method to LMRM. We have subsequently proved its membership in the NDRMI class. Furthermore, we applied mixture of Gaussians likelihood when the prior is LMRM and showed in detail under the setting where its direction is in a form of a Gamma measure. This can be seen as a multi-variational Dirichlet Process(es).

In terms of experiments, we showed the performance and the superiority of our model from both synthetic and clinical data. In particular, we systematically evaluate our model under different $R$, i.e., the number of the assumed CRMs. It can be seen that when $R$ is greater or equal to the ground-truth, the inferred clustering information largely agrees with the true structure of the data samples, and across the board in all groups. We also noted that the clusters will split up or mixed together when $R$ is smaller than of the ground-truth.

In this paper, we discussed the LMRM when the directions are assumed to be Gamma Processes. However, the $\sigma$-stable Process, the generalized Gamma Process are believed to have a more stable property. For example, the power law of the partition functions \cite{Lijoi2007}, which is believed to be more suitable in real world scenarios. By using some straightforward mathematical derivations, it is easy to derive the functions functions $\tau_{\boldsymbol{q}_k}(u_1,...,u_d)$ and $\psi(u_1,...,u_d)$ and hence we can compare the performance of these processes with respect to that of  the Gamma Process. More generally, the L\'evy copula should be studied since we can define more general L\'evy measures for dependent random measures by virtual of it. The hierarchical Dirichlet process is another model to share atoms across groups through a common base measure. In the future, we will study the interesting mathematical relationship between HDP and the dependent random measures.

\section{Acknowledgements}

This work was supported by 973 Program of China [grant numbers 2014CB340401]; International Exchange Program for Graduate Students, Tongji University.

\bibliographystyle{plain}
\bibliography{My_Collection}

\begin{appendices}
\begin{proof}[Proof of Proposition \ref{prop}]
We rewrite equation (\ref{eq:EPPF_before}) to the form
\begin{align}\label{eq:before_prop}
\mathbb{P}(\mathbf{X}, Y_k \in C_{k}, ~ k=1,...,K)= \int_{(\mathbb{R}^+)^d}\mathbb{E}\left[ \prod_{i=1}^{d} \frac{u_i^{n_i-1}}{\Gamma(n_i)}  e^{-u_iT_i}  \prod_{k=1}^{K} \tilde{\mu}_i(C_k)^{q_{i,k}}\right] du_1,...,du_d ,
\end{align}
and split $\mathbb{X} =\cup_{k=0}^K C_k $, where $C_0 = \mathbb{X} - \left(\cup_{k=1}^K C_k \right)$. Then the total measure $T_i$ is decomposed into $T_i = \tilde{\mu}_i(\mathbb{X}) = \sum_{k=0}^{K} \tilde{\mu}_i(C_k)$. By the independent assumption for disjoint measurable sets $A_1,...,A_n$ that the random measures $\tilde{\mu}_{i_1}(A_1),...,\tilde{\mu}_{i_n}(A_n)$ are mutually independent no matter $i_j$ are equal or not for $j = 1,...,n$. The expectation in equation (\ref{eq:before_prop}) is changed to
\begin{align*}
\mathbb{E}\left[ \prod_{i=1}^{d} \frac{u_i^{n_i-1}}{\Gamma(n_i)}  e^{-u_iT_i}  \prod_{k=1}^{K} \tilde{\mu}_i(C_k)^{q_{i,k}}\right] &= \left( \prod_{i=1}^{d} \frac{u_i^{n_i-1}}{\Gamma(n_i)} \right) \mathbb{E}\left[e^{-\sum_{i=1}^{d}u_i \tilde{\mu}_i(C_0)} \right] \times \\
&~~~~~~~~~~~~~~~~~~~~~~~~~~\prod_{k=1}^{K} \mathbb{E}\left[e^{-\sum_{i=1}^{d} u_i \tilde{\mu}_i(C_k)}\left( \prod_{i=1}^{d} \tilde{\mu}_i(C_k)^{q_{i,k}}\right)\right].
\end{align*}
We apply the multivariate Fa\`a di Bruno formula \cite{Constantine1996} on the last term to get
\begin{align*}
&\mathbb{E}\left[e^{-\sum_{i=1}^{d} u_i \tilde{\mu}_i(C_k)}\left( \prod_{i=1}^{d} \tilde{\mu}_i(C_k)^{q_{i,k}}\right)\right] = \\
&~~~~~~~~~~~~~~~=(-1)^{n_i} \frac{\partial^{n_i}}{\partial u_1^{q_{1,k}}\cdots \partial u_K^{q_{i,K}}} \exp\left\{-\int_{\mathbb{X}\times (\mathbb{R}^+)^d}(1 - e^{-\sum_{i=1}^{d}s_i})\nu(ds_1,...,ds_d)H(C_k) \right\} \\
&~~~~~~~~~~~~~~~= \exp\left\{-\int_{\mathbb{X}\times (\mathbb{R}^+)^d}(1 - e^{-\sum_{i=1}^{d}s_iu_i})\nu(ds_1,...,ds_d)H(C_k) \right\} \\ &~~~~~~~~~~~~~~~~~~~~~~~~~~~~~~~~~~~~~~~~~~~~~~\times \left[\int_{\mathbb{X}\times (\mathbb{R}^+)^d}e^{-\sum_{i=1}^{d}s_iu_i}\prod_{i=1}^{d}s_i^{q_{i,k}}\nu(ds_1,...,ds_d)H(C_k) + o(H(C_k))\right]
\end{align*}
The second equation follows from the Fa\`a di Bruno formula and equation (\ref{eq:def_drmi}). Then the conclusion follows by setting $C_k \coloneqq C_{k,\epsilon} = \{y:d(Y_k,y) < \epsilon\}$ and letting $\epsilon \downarrow 0$ and inserting it back into equation (\ref{eq:before_prop}).
\end{proof}
\end{appendices}

\end{document}